\documentclass[12pt]{article}

\usepackage{endnotes}
\let\footnote=\endnote
\usepackage[running]{lineno}
\usepackage{sectsty}
\usepackage{graphicx,amssymb,amsmath,amsthm,cite,cases,color}
\usepackage{verbatim,bm}
\def\be{\begin{equation}}
\def\ee{\end{equation}}
\def\ben{\begin{eqnarray}}
\def\een{\end{eqnarray}}

\def\ds{\displaystyle}

\def\Ga{\mathcal{L}}
\def\fl{\mathrm{fl}}

\def\D{\mathcal{D}}
\def\Dt{\mathcal{\til{D}}}
\def\R{\mathbb{R}}
\def\N{\mathbb{N}}
\def\V{\mathbb{S}}
\def\S{\mathcal{S}}

\def\til{\tilde}

\def\topc{{{\top}}}

\def\bel{\begin{linenomath}}
\def\eel{\end{linenomath}}

\def\vI{\mathbf{I}}
\def\vA{\mathbf{A}}

\def\vT{\mathbf{T}}
\def\vG{\mathbf{G}}
\def\vHk{\mathbf{H}_k}
\def\vGk{\mathbf{G}_{k}}

\def\vS{\mathbf{S}}

\def\vUk{\mathbf{U}_{k}}

\def\vLk{\mathbf{\Lambda}_{k}}

\def\vd{\mathbf{d}}

\def\vdt{\mathbf{\til{d}}}

\def\vf{\mathbf{f}}

\def\vg{\mathbf{g}}

\def\vA{\mathbf{A}}

\def\vr{\mathbf{r}}

\def\O{\operatorname{O}}
\def\Range{\operatorname{Range}}

\def\vP{\mathbf{P}}
\def\op{\hat{\operatorname{\vP}}}
\def\oj{\hat{\operatorname{J}}}
\def\kq{k_q}
\def\qs{q^\star}
\def\kqs{k_{\qs}}

\def\Nb{N_b}

\newcommand{\la}{\left\langle}
\newcommand{\ra}{\right\rangle}
\newcommand{\eps}{\epsilon}

\newtheorem{proposition}{Proposition}

\newtheorem{lemma}{Lemma}
\newtheorem{theorem}{Theorem}

\def\Nx{N_x}
\def\Ny{N_y}

\newcommand{\Spann}{{\mbox{\rm{span}}}}

\begin{document}
\title{Analysis of the 
Self Projected Matching Pursuit Algorithm}
\author{Laura Rebollo-Neira\\
Mathematics Department, Aston University\\
B4 7ET, Birmingham, United Kingdom\\\\
Miroslav Rozlo\v zn\' ik\\
Institute of Mathematics, Czech Academy of Sciences\\
{\u{Z}}itn{\'a} 25, CZ -- 115 67 Praha 1, Czech Republic\\\\
Pradip Sasmal\\ 
Department of Electrical Communication Engineering\\
Indian Institute of Science, Bangalore\\
Karnataka, 560012, India} 
\maketitle
\begin{abstract}
The convergence and numerical analysis of a low memory 
implementation of the Orthogonal Matching Pursuit 
greedy strategy, 
 which is termed Self Projected Matching Pursuit,
is presented. This approach renders an
 iterative way of solving the least 
squares problem with much less storage requirement 
than direct linear algebra techniques. Hence, it is  
 appropriate for solving large linear systems. 
The analysis highlights its suitability within the 
class of well posed problems.
 
\end{abstract}
\section{Introduction} 
Sparse representation refers to the process by which
a signal is transformed
 in order to reduce its dimensionality.
Traditional methods implement the
transformation using fast orthogonal transforms.
Higher levels of sparsity are attained,
in many cases,
 if the transformation is carried out using
a large redundant set called  a dictionary.
For the most part this method is implemented by 
   minimization
of the $l_1$-norm \cite{CDS01,DT05,Eld10} and the
so-called greedy strategies. The latter consist in
 adaptively constructing a signal representation
 as a linear superposition of elements
taken from the dictionary. In this contribution
we focus on the analysis of a low memory implementation
of a particular method within this category.

 Greedy strategies have been the subject of extensive 
research in the last two decades \cite{MZ93,PRK93, 
Nat95,DT96, VNT99, RNL02, ARNS04, 
ARN06, Tro04, GV06, DTD06,BD08, NT09}
and currently support diverse applications
\cite{LSL14, YRV16, YLC19}.  
The simplest, yet very effective greedy algorithm 
for the sparse representation of large signals,  
was introduced to the 
signal processing community in \cite{MZ93} with the name 
of Matching Pursuit (MP). It had 
previously appeared as a regression technique in
statistics \cite{FS81,Jon87},
where the convergence property was established.
While MP converges asymptotically to a signal in the 
linear span of the dictionary, or to its orthogonal 
projection if the signal is out of that space, 
  the approach is not stepwise optimal because it does not 
yield an orthogonal projection at each 
step.
A refinement to MP which  fulfils this requirement 
is refereed to as Orthogonal Matching Pursuit (OMP) 
\cite{PRK93}. 
If implemented by direct methods the OMP approach is 
 very effective 
 up to some dimensionality. When processing 
large signals, however, the storage requirements
 frequently exceed the memory capacity of a standard  
 computer.
An alternative implementation of OMP, 
which requires 
much less memory than direct implementations  
 is considered in \cite{RNB13}.
The approach is termed Self Projected Matching Pursuit 
(SPMP). It produces the orthogonal projection of the  signal,
at each iteration, by applying MP 
using a sub-dictionary consisting only of
the already selected elements.  
 A convenient feature of SPMP when applied in 2D 
(SPMP2D)  \cite{RNB13,LRN17} and 
 3D  (SPMP3D) \cite{RNW19} is that it fully exploits
 the separability of dictionaries.
 Nevertheless, until now the method had not been
 analyzed. Thus, the main contributions of this paper are: 
\begin{itemize}
\item
The convergence analysis of the SPMP approach, which deals with those cases where the standard implementation of the 
OMP method is not feasible due to storage requirements. 
\item
The error analysis of the SPMP approach.
\end{itemize}
Additionally, the SPMP approach is extended to consider its
 Hierarchized Block Wise (HBW) version for
approximating a signal partition subjected to
 a global constraint on sparsity.

The paper is organized as follows: Sec.~\ref{SPMP}
 recalls the SPMP algorithm,  proves 
 the power law for the  convergence  rate 
of the self projection step 
  and  develops its  numerical analysis. 
 In Sec.~\ref{HBWSPMP}  the applicability of the method 
 is extended by dedicating the algorithm to the 
 approximation of non stationary signals by partitioning. 
The final conclusions are presented in Sec.~\ref{Con}. 

\section{Self Projected Matching Pursuit (SPMP)}
\label{SPMP}
Throughout the paper $\R$  and $\N$
represent the sets of real 
and natural numbers, respectively. 
Boldface fonts  are used to indicate Euclidean vectors
or matrices and standard mathematical fonts  to
indicate components,  e.g., $\vd \in \R^N$ is a vector of 
$N$-components $d(i) \in \R\,, i=1,\ldots,N$
and $\vA \in \R^{\Nx \times \Ny}$ a matrix of elements $A(i,j) \in \R\,,\,i=1,\ldots, \Nx, \, j=1,\ldots,\Ny$. The  transpose of $\vA$ is denoted as $\vA^\topc$.
The operation 
$\la \cdot,\cdot \ra$ indicates the Euclidean inner
product and  $\| \cdot \|$ the induced norm, i.e. 
$\| \vd \|^2= \la \vd, \vd \ra$, with the usual 
inner product definition: For $\vg \in \R^N$ 
and $\vf \in \R^N$
\be
\la \vf, \vg \ra = \sum_{i=1}^N f(i) g(i). 
\ee
Let's consider a finite set $\D$ of $M$ 
of normalized 
 vectors 
$\D=\{\vd_n \in \R^{N}\,; \|\vd_n\|=1\}_{n=1}^M$ 
 and let's define $\V_M= \Spann(\D)$, which could be 
 $\R^N$. For $M >\text{dim}(\V_M)$    
 the set $\D$ is a redundant dictionary  
 and the elements are called {\em{atoms}}. 
Given a signal, as a vector $\vf \in \R^{N}$, the 
$k$-term {\em{atomic decomposition}} for its
approximation takes the form
\be
\label{atoq}
\vf^{k}= \sum_{j=1}^{k}
c(j) \vd_{\ell_j}.
\ee
The problem of how to select from $\D$ the smallest number
of $k$ 
atoms $\vd_{\ell_j},\,j=1\ldots,k$, such that 
 $\|\vf^k - \vf\| < \rho$,  for a given tolerance 
parameter $\rho$, is an NP-hard problem
\cite{Nat95}. In practical applications 
 one looks for `tractable sparse' solutions.
This is to say 
a representation involving a number of $k$-terms, 
 with $k$ acceptably small in relation to $N$.
The simplest approach to tackle this problem is MP. It 
evolves by successive approximations as follows 
\cite{MZ93}:
 Setting $k=0$ and starting with an initial approximation
$\vf^0=0$ and residual $\vr^{0} = \vf$, 
the algorithm
  progresses  by sub-decomposing the $k$-th order residual 
in the form 
\be
\vr^{k} =
\la \vd_{\ell_{k+1}}, \vr^{k} \ra \vd_{\ell_{k+1}}
+ \vr^{k+1},
\label{tech:1}
\ee
where $\vd_{\ell_{k+1}}$ is the atom corresponding to the index selected as 
\be
\label{selMP}
\ell_{k+1}=  \operatorname*{arg\,max}_{\substack{n=1,\ldots,M}} |\la \vd_{n} , \vr^{k} \ra|.
\ee
This atom is used to update the approximation $\vf^k$ as
\be
\label{upfk}
\vf^{k+1} = \vf^{k} + \la \vd_{\ell_{k+1}}, \vr^{k} \ra  \vd_{\ell_{k+1}}.
\ee
From \eqref{tech:1} it follows that 
$\|\vr^{k+1}\| \le \|\vr^{k}\|$, since 
\be
\|\vr^{k}\|^{2} = |\la \vd_{\ell_{k+1}}, \vr^{k} \ra|^{2} + \|\vr^{k+1}\|^{2}.
\label{tech:2}
\ee
\begin{lemma}
\label{lem}
In the limit $k \rightarrow \infty$, the sequence  $\vf^{k}$
given in \eqref{upfk} 
converges to $\vf$, if $\vf \in \V_{M}$, or to  
$\op_{\V_{M}}\vf$, 
the orthogonal projection of $\vf$ onto
$\V_{M}$, if $\vf \notin \V_{M}$.
\end{lemma}
This lemma is just a particular case of the
 well established and more general convergence  results  for  MP \cite{Jon87,MZ93,VNT99}.
However, for pedagogical reasons, due to 
its crucial importance for this work, 
 we present here a particular proof
holding {\em{only}} for finite dimension spaces which,
for this reason, is very simple.
\begin{proof}
We notice, from \eqref{tech:2}, that $\|\vr^{k}\|^2$ 
is a decreasing
sequence  which, since  $\|\vr^{k}\|^2\ge 0$ for all $k$, 
is bounded.
It is a classic result of analysis that a 
decreasing and bounded sequence 
converges to the infimum \cite{BS99},  i.e.,
$\lim_{k\to \infty} \|\vr^{k}\|^{2}= b$.
We prove next that $b=0$.
Since
$$\|\vr^{k+1}\|^{2}= \|\vr^{k}\|^{2} - |\la \vd_{\ell_{k+1}}, \vr^{k} \ra|^{2},$$  taking $\lim_{k \to \infty}$ of both sizes, we have:
$$b^2=b^2 - \lim_{k \to \infty}  |\la \vd_{\ell_{k+1}}, \vr^{k} \ra|^{2}.$$
Thus, $\lim_{k \to \infty}  |\la \vd_{\ell_{k+1}}, 
\vr^{k} \ra|=0,$ which using \eqref{selMP} 
implies $\lim_{k \to \infty} |\la \vd_{n},  \vr^{k} \ra|=0, 
\, n=1,\ldots,M.$
Consequently, 
either $\lim_{k \to \infty} \vr^{k}=0$ or, if the 
dictionary is incomplete, 
$\lim_{k \to \infty} \vr^{k} $ is orthogonal to all
the elements in $\D$.
This result is readily obtainable here, because of 
the finite dimension framework. 
Indeed, in finite dimension 
 the existence of a  
reciprocal (dual) dictionary 
$\Dt=\{\vdt_{n} \in \R^N\}_{n=1}^M$  spanning the 
same space as $\D$  
 is guaranteed \cite{LRN06,CKP12}. Hence, even if due to the 
redundancy of $\D$ the 
decomposition is 
not unique, all $\vg  \in \V_M=\Spann(\D)=
\Spann(\Dt)$ can be decomposed in the form
$$
\vg=\sum_{n=1}^M \vd_{n} \la \vdt_{n}, \vg \ra= \sum_{n=1}^M \vdt_{n} \la \vd_{n}, \vg \ra.
$$
Furthermore, every vector in 
$\R^N$, and in particular $\vr^{k}$, can be 
split as $\vr^{k}= \op_{\V_M} \vr^{k} + 
\op_{\V_M^\perp} \vr^{k}$, 
where $\op_{\V_M} \vr^{k}$ is the orthogonal projection 
onto $\V_M$ and  
 $\op_{\V_M^\perp} \vr^{k}$ is the orthogonal projection
 onto the subspace $\V_M^\perp$, which is 
 the orthogonal complement of  $\V_M$ 
  in $\R^N$.
From the relation
$$\op_{\V_M} \vr^{k}=\sum_{n=1}^M \vd_{n} \la \vdt_{n}, \vr^{k} \ra= \sum_{n=1}^M \vdt_{n} \la \vd_{n}, \vr^{k}\ra,$$
and because it involves a {\em{finite}} sum, 
 we conclude that  $\lim_{k \to \infty}  
|\la \vd_{n}, 
\vr^{k} \ra|=0,\,n=1,\ldots,M  
\implies\,
\lim_{k \to \infty} \op_{\V_M}  \vr^{k}=0$. Then,
either 
$ \lim_{k \to \infty} \vr^{k}=0$ or
$\lim_{k \to \infty} \vr^{k} \in {\V_M^\perp}$. Consequently, since 
$\vf^k= \vf - \vr^{k} \in \V_M$, it follows that 
 $\lim_{k \to \infty} \vf^k=  \op_{\V_M} \vf$.
\end{proof}
\subsection{Adding Self Projections}
\label{asp}
The obvious way of improving the MP algorithm is to calculate the coefficients in \eqref{atoq} so as to 
minimize the norm of the 
 residual error $\|\vf -\vf^k\|$  
 for every value of $k$.  In other words, to require that, 
at each 
iteration, the coefficients in \eqref{atoq} should fulfill the 
condition $\vf^k = \op_{\V_k} \vf$, where 
${\V_k}=\Spann\{\vd_{\ell_j}\}_{j=1}^k$. Hence the name, OMP,  
 of the approach achieving this.
  When the dimension of the problem 
is such that memory requirement is not an issue, a 
number of convenient direct linear algebra methods 
for performing  
the projection $\op_{\V_k} \vf$ are
available \cite{Bjo96,GV96,Hig02}. However, 
  it is the need of 
 calculating  orthogonal projections with much less 
storage demands than direct methods what originated 
the SPMP approach described below.

SPMP relays on Lemma \ref{lem} to realize the 
orthogonal projection step and produces an
alternative iterative implementation of the OMP approach.
Given a signal $\vf$, a tolerance error $\rho$ for the approximation, and a dictionary $\D$, the SPMP algorithm  
proceeds as follows \cite{RNB13}: 
Set $\Ga_0=\{\emptyset\}$, $\vf^0=0$ and $\vr^0=\vf$. 
Starting from $k=0$, 
at each iteration implement the steps below.
\begin{itemize}
\item[i)]
 While $\|\vr^k\| > \rho$ increment $k \leftarrow k+1$ and apply 
the  MP criterion for selecting from $\D$ the 
 atom $\vd_{\ell_{k}}$ to be placed in the atomic decomposition i.e., select $\ell_{k}$ such that
\be
\label{selMPal}
\ell_{k}=  \operatorname*{arg\,max}_{\substack{n=1,\ldots,M}} |\la \vd_{n} , \vr^{k-1} \ra|. 
\ee
Update the set 
 $\Ga_k = \Ga_{k-1} \cup \{\ell_{k}\}$.
Compute $c(k)=\la \vd_{\ell_{k}} , \vr^{k-1}\ra$, 
update the approximation of $\vf$ as 
 $\vf^k = \vf^{k-1} + 
 c(k) \vd_{\ell_{k}}$, 
and evaluate the new residual $\vr^k= \vf -\vf^k$.
\item[ii)]
Realize the orthogonal projection 
by subtracting from $\vr^k$ the component in 
$\V_k= \Spann\{\vd_{\ell_{i}}\}_{i=1}^k$, via the MP algorithm, as follows. 
Let $\eps$ be a given tolerance for the projection 
error. 
  Set $j=1$, $\vr^{k,0}=\vr^{k}$ and at iteration 
$j$ implement the 
steps below:
\begin{itemize}
\item [(a)]
Choose,   
out of the set $\Ga_k$, the index $l_j$ such that
$$l_j= \operatorname*{arg\,max}_{\substack{i=1,\ldots,k}} \left|\la \vd_{\ell_{i}}, \vr^{k,j-1} \ra\right|.$$
If $\left|\la \vd_{l_j}, \vr^{k,j-1} \ra\right|< \eps$ 
set $\vr^k \leftarrow \vr^{k,j-1}$ and return to i).
 Otherwise continue with steps (b) and (c) as follows. 
\item [(b)]
Use $\la \vd_{l_j} , \vr^{k,j-1} \ra$  to 
update the coefficient $c(l_j)$, the approximation 
$\vf^k$, and the residual,  as
\ben
c(l_j)  &\leftarrow & c(l_j)+ \la \vd_{l_j}, \vr^{k,j-1} \ra, \nonumber \\
\vf^k &\leftarrow & \vf^k + \la \vd_{l_j}, \vr^{k,j-1}\ra \vd_{l_j},\nonumber\\
\vr^{k,j}&=& \vr^{k,j-1} -  \la \vd_{l_j}, \vr^{k,j-1} \ra \vd_{l_j}. \nonumber
\een
\item [(c)]
Increment $j \leftarrow j+1$ and repeat steps 
(a) $\to$ (c) until the stopping criterion is met. 
\end{itemize}
\end{itemize}
As proved in Lemma~\ref{lem}, 
 by means of the self-projections implemented 
by steps (a) -- (c), 
at each iteration $k$ the SPMP algorithm 
asymptotically delivers 
an approximation $\vf^k=\op_{\V_k} \vf$  
with residual $\vr^k= \vf - \op_{\V_k} \vf$. 
The next Lemma stresses the fact that, as a consequence, 
 the SPMP algorithm 
selects only linearly independent atoms.
\begin{lemma}
\label{lem2}
If the atoms $\vd_{\ell_i},\,i=1,\ldots,k$
are selected by criterion \eqref{selMPal}, and the
residual $\vr^k$ is refined by self projections at each
iteration, the selected atoms
 constitutes a linearly independent set.
\end{lemma}
\begin{proof}
For $k=1$  the
lemma is triviality true. Assuming that it is true
 for the first $k$ atoms we prove that it is
true for $k+1$ atoms.

Suppose, on the contrary,
that $\left|\la \vd_{\ell_{k+1}}, \vr^k \ra \right| >0$ and
$\vd_{\ell_{k+1}}= \sum_{i=1}^k a_i \vd_{\ell_i}$,
where $a_i,\, i=1,\ldots,k$ are
 numbers such that $\sum_{i=1}^k |a_i|^2 >0$.
Since at the iteration $k$ the SPMP algorithm asymptotically 
gives 
 a residual that satisfies $\vr^k= \vf - \op_{\V_k} \vf$ 
we have:
$$\la \vd_{\ell_{k+1}}, \vr^k \ra= 
\la \sum_{i=1}^k a_i \vd_{\ell_i}, \vf - \op_{\V_k} \vf\ra=0,$$
which contradicts the assumption that 
$\left|\la \vd_{\ell_{k+1}}, \vr^k \ra \right|>0$. 
It is concluded 
then that $\vd_{\ell_{k+1}}$ cannot be expressed 
as a linear combination of the previously 
selected atoms. 
\end{proof}

\subsection{Convergence rate of the self projection
steps}
\label{cr}
We start by recalling some properties of symmetric 
matrices, which will be used for the analysis.
Let the atoms $\vd_{\ell_i},\,i=1,\ldots,k$
be the columns of the matrix $\vS_{k}$. 
Since the atoms are linearly independent, the 
symmetric 
 matrix $\vHk=\vS_k \vS_k^\topc$ has $k$ nonzero 
eigenvalues, which are also the $k$ eigenvalues of the 
Gram matrix $\vG_k= \vS_k^\topc \vS_k$. In
 terms of the corresponding 
 eigenvectors $\vHk$ can be expressed
as
\be
\label{gg}
\vHk=\vUk \vLk \vUk^\topc, 
\ee
where $\vLk$ is a diagonal matrix, containing
in the diagonal its
eigenvalues $\lambda^{k}_i>0,\,i=1,\ldots,k$
in descending order. Since all the atoms are
normalized, it holds that
$$\text{Trace}(\vHk)= \sum_{i=1}^k \lambda^k_i =k.$$
This relation 
implies that $ k\lambda^{k}_k \le k \le  k  \lambda^{k}_1$,
which ensures that $\lambda^{k}_k \le 1$. 
The columns of matrix $\vUk$  are the 
normalized eigenvectors of $\vHk$ corresponding to the 
eigenvalues $\lambda^{k}_i>0,\,i=1,\ldots,k$. Since 
$\vHk$ is symmetric these eigenvectors constitute an 
 orthonormal basis for $\V_{k}=\Range(\vS_k)$. 
 Accordingly,
the orthogonal projector $\op_{\V_{k}}$
admits a representation of the form:
\be
\label{pp}
\op_{\V_{k}} = \vUk \vUk^\topc.
\ee
Then, the following inequality 
 arises from \eqref{gg} and \eqref{pp}, 
\be
\label{lowb}
\|\vS_k^\topc \vg \|^2 =  
\la \vg, \vS_k \vS_k^\topc\vg \ra 
\ge \lambda^{k}_{k} \|\op_{\V_{k}} \vg\|^2,\quad \forall\, \vg \in \R^N.
\ee
This inequality will be used for the analysis of 
the convergence 
rate of the self-projection step.
\begin{proposition}
At iteration $j$ the component in $\V_{k}$ of the 
residual $\vr^{k,j}$ is bounded as 
\be
\label{crprop}
\|\op_{\V_{k}}\vr^{k,j}\|^2 \le
\left(1-\frac{\lambda^{k}_{k}}{k}\right)^{j}
\|\vr^{k,0}\|^2.
\ee
\end{proposition}
\begin{proof}
Let's recall that the projection step operates by setting 
$\vr^{k,0}=\vr^{k}$ and at the 
$j$-th iteration decomposing 
the residual $\vr^{k,j}$  as 
\be
\vr^{k,j} = {\vr^{k,j-1}} -
\la \vd_{l_j}, \vr^{k,j-1} \ra \vd_{l_j},
\label{techsp:1}
\ee
where 
\be
\label{selMP2}
l_j= \operatorname*{arg\,max}_{\substack{i=1,\ldots,k}} |\la \vd_{\ell_i} , \vr^{k,j-1} \ra|.
\ee
Since $\op_{\V_{k}} \vd_{l_j}= \vd_{l_j}$, applying the operator $\op_{\V_{k}}$ on both sides
of \eqref{techsp:1} we have,
$$\op_{\V_{k}}\vr^{k,j}= 
\op_{\V_{k}}\vr^{k,j-1}-\left\langle \vd_{l_j},\vr^{k,j-1} \right\rangle \vd_{l_j},$$
and consequently  
\begin{equation}
\|\op_{\V_{k}}\vr^{k,j}\|^{2}= 
\|\op_{\V_{k}}\vr^{k,j-1}\|^2-|\left\langle {\vr^{k,j-1}},
\vd_{l_j}\right\rangle|^2.
\end{equation}
By definition of the index $l_j$ 
(cf.\eqref{selMP2}),   
and using \eqref{lowb}, we assert that 
$$|\left\langle \vd_{l_j}, \vr^{k,j-1}\right\rangle|^2
\ge \frac{1}{k} 
\sum_{i=1}^{k}|\la \vd_{i}, \vr^{k,j-1} \ra| =
\frac{1}{k} \| \vS_k^\topc \vr^{k,j-1}\|^2
\ge \frac{\lambda^{k}_k}{k}\|\op_{\S_{k}}\vr^{k,j-1}\|^2.$$
Then, we finally obtain
\be
\label{cr1}
\|\op_{\V_{k}}\vr^{k,j}\|^2 \le 
\left(1-\frac{\lambda^{k}_k}{k}\right)
\|\op_{\V_{k}}\vr^{k,j-1}\|^2,
\ee
and applying the inequality back $j$-times
\be
\label{cr2}
\|\op_{\V_{k}}\vr^{k,j}\|^2 \le
\left(1-\frac{\lambda^{k}_{k}}{k}\right)^{j}
\|\op_{\V_{k}}{\vr^{k,0}}\|^2 \le
\left(1-\frac{\lambda^{k}_{k}}{k}\right)^{j}
\|\vr^{k,0}\|^2.
\ee
\end{proof}
The  bound  \eqref{cr2} 
gives a power form for the  worst-case 
convergence rate to a residual vector having no component in 
$\V_{k}$. It also shows the dependence of the 
 convergence rate on the smallest eigenvalue 
of the Gram matrix $\vGk$ 
of the selected atoms up to iteration $k$. According 
to the interlacing theorem (\!\cite{HJ91}, p 189--190)  
 it is true that  $\lambda^{k+1}_{k+1}<\lambda_k^k$. 
Hence,  
in general one could expect the convergence 
rate of the self projection to slow 
down as the iterative selection of atoms progresses.

{\bf{Remark 1:}} The  
convergence of MP in terms of the dictionary's 
coherence \cite{Tro04} is derived in \cite{GV06} 
for the case of 
quasi incoherent dictionaries. That condition is 
 too stringent for signals
of practical interest, 
 which are far more compressible when using a
 highly coherent dictionary than when using an
orthogonal or quasi orthogonal basis. 
Contrarily, the 
expression \eqref{cr2} gives a realistic appreciation
 with respect to the broad range of effective applicability
of the SPMP approach. Regardless of the dictionary 
coherence, SPMP can be an effective  
low memory implementation of the OMP greedy strategy  
as long as the least squares problem, for 
the determination of the coefficients in the decomposition 
 \eqref{atoq}, is a well posed problem. 
\subsection{Numerical Example I}
\label{NE1}
We illustrate here some features of the 
 numerical convergence of the SPMP method in 
relation to the particular application to sparse 
signal decomposition.  

The quality of the  $k$-term approximation $\vf^k$ of  a
signal $\vf$  is assessed by the
Signal to Noise Ratio (SNR), which is defined as
$$\text{SNR}=10 \log_{10} \frac{\|\vf\|^2}{\|\vf - \vf^k\|^2}.$$
As an example  we approximate, up to SNR = 35 dB, the $N=1024$
samples of a music signal shown on Fig.~\ref{pia1024}.
This  SNR value produces a high quality approximation of 
the signal,  indistinguishable from the original signal 
in the scale of  Fig.~\ref{pia1024}.

\begin{figure}[!ht]
\begin{center}
\includegraphics[width=10cm]{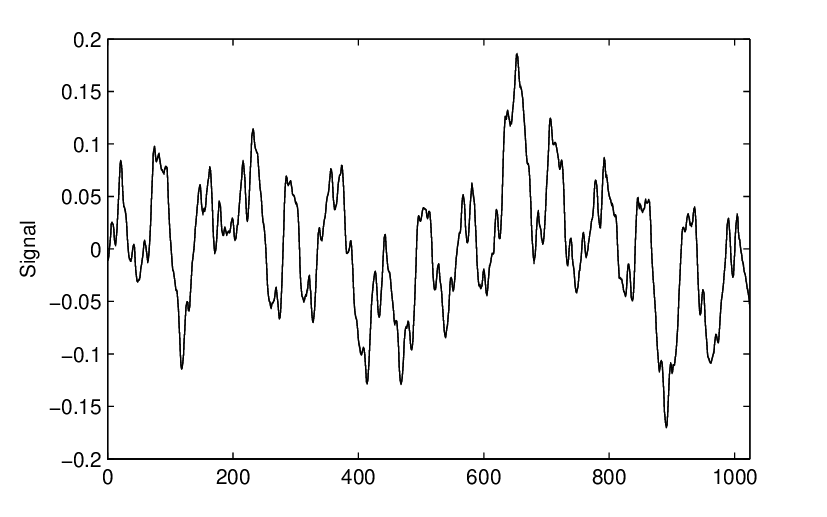}
\end{center}
\caption{$N=1024$ samples of a music signal and its 
approximation.}
\label{pia1024}
\end{figure}

In the first instance we consider a uniformly 
random dictionary with redundancy four, which is 
certainly not an appropriate dictionary for representing music.
Indeed, the SPMP method requires $k=648$ atoms for  approximating the 1024 samples up to SNR = 35 dB.
The left graph in Fig.~\ref{karantri} shows the number 
of iterations spent 
in the orthogonal projection step vs the number of 
atoms involved in the corresponding step. 

\begin{figure}[!ht]
\begin{center}
\includegraphics[width=8.1cm]{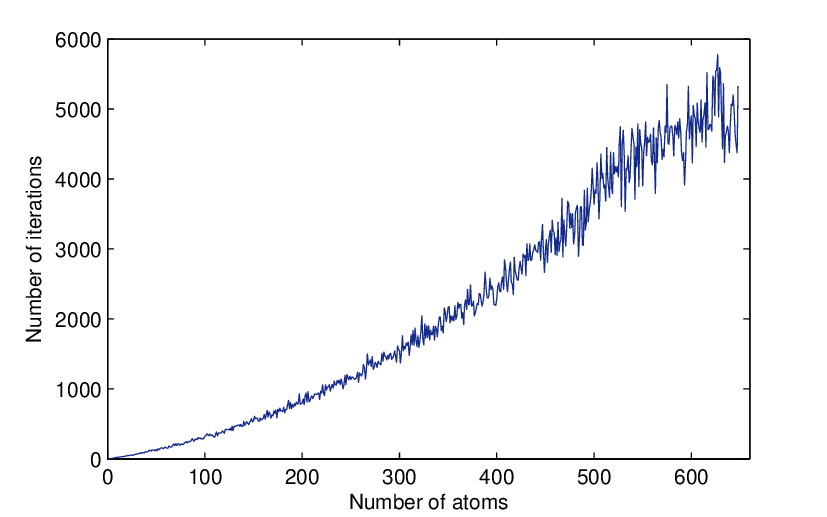}
\includegraphics[width=8.1cm]{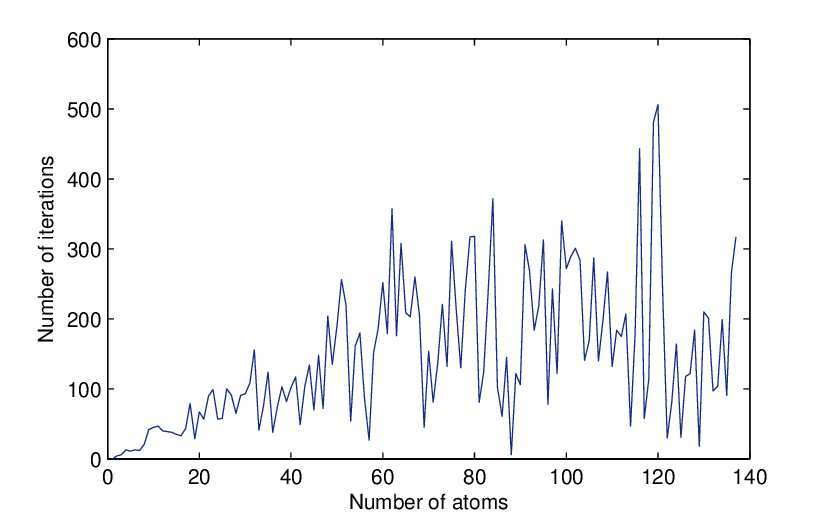}
\end{center}
\caption{The left graph shows the number of iterations 
needed by SPMP for the approximation of the signal in 
Fig.~\ref{pia1024}  using a random 
dictionary.  The right graph has the same description 
as the left graph but using the trigonometric dictionary 
$\D^{cs}$.}
\label{karantri}
\end{figure}

In order to obtain a sparse representation 
of the same signal we now change the 
 random dictionary to the trigonometric
one, $\D^{cs}= \D^c \cup \D^s$,
with $\D^c$ and $\D^s$ as given below
\be
\label{RDC}
\mathcal{D}^{c}=\{{w^c(n)}\,
 \cos ({\frac{{\pi(2i-1)(n-1)}}{2M}}),\,i=1,\ldots,N\}_{n=1}^{M}.
\ee
and
\be
\label{RSC}
\mathcal{D}^{s}=\{{w^s(n)}\, \sin  ({\frac{{\pi(2i-1)n}}{2M}}), \,i=1,\ldots,N\}_{n=1}^{M},
\ee
where  $w^c(n)$ and $w^s(n)$ are normalization factors.
Taking $M=2N$ the dictionary $\D^{cs}$ has the 
same redundancy as the previous one, but is suitable 
for representing music. The SPMP method uses now
 137 atoms for approximating the 
signal in Fig.~\ref{pia1024} up to SNR = 35 dB 
(the same number of atoms the 
OMP method needs). The right graph in 
 Fig.~\ref{karantri} shows
the iterations needed by the orthogonal 
 projection step with dictionary $\D^{cs}$. 
 It is clear that, up to the same numerical precision, 
 the iterations to achieve the orthogonal 
projection depend on the dictionary.
 
Next we keep using the dictionary $\D^{cs}$   
 for tackling the following large dimension
problem: The representation by non-orthogonal frequency
components of the flute tone depicted in Fig.\ref{tone}, 
which consists of $N=61285$ samples.
\begin{figure}[!ht]
\begin{center}
\includegraphics[width=9cm]{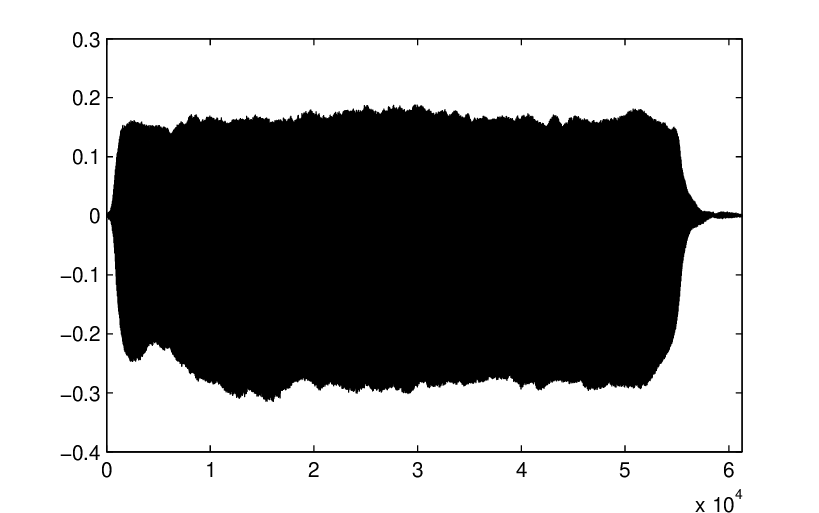}
\end{center}
\caption{Baroque flute tone C$\#5$.
Sound clip Csharp5.baroque.wav
available on  {\tt{https://newt.phys.unsw.edu.au/music/flute/baroque/Csharp5.baroque.html}}}
\label{tone}
\end{figure}
A particularity of dictionary
$\D^{cs}$ is that,
because by padding with zeros
the inner products with its elements
can be computed via the Fast Fourier Transform
 \cite{LRN16,RNA16}, 
there is no need to store
the dictionary as such (otherwise in this example
 it  would be a matrix of
dimension $61285 \times 245140$). 
The left graph of Fig.~\ref{katri2} shows the number 
of iterations vs the number of atoms in the signal 
approximation. The right graph is the histogram 
of the values in the left graph. The mean 
value of the number of iterations in the whole 
approximation is 44. 

\begin{figure}[!ht]
\begin{center}
\includegraphics[width=8.1cm]{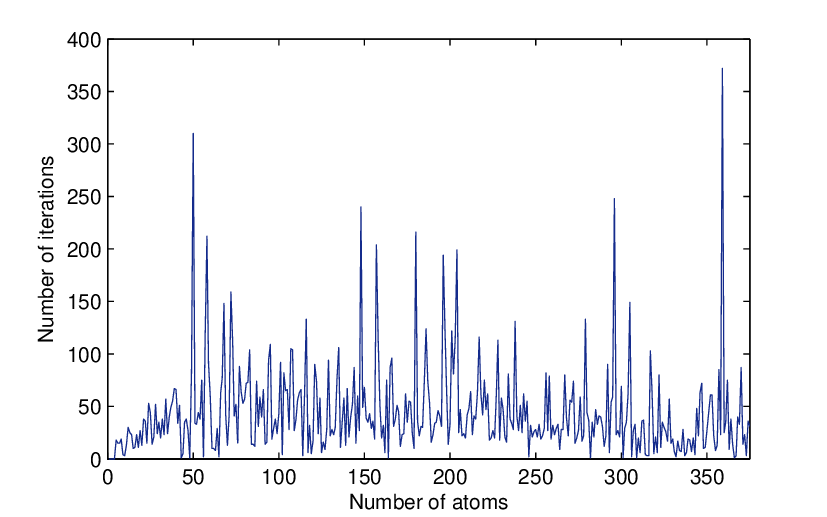}
\includegraphics[width=8.1cm]{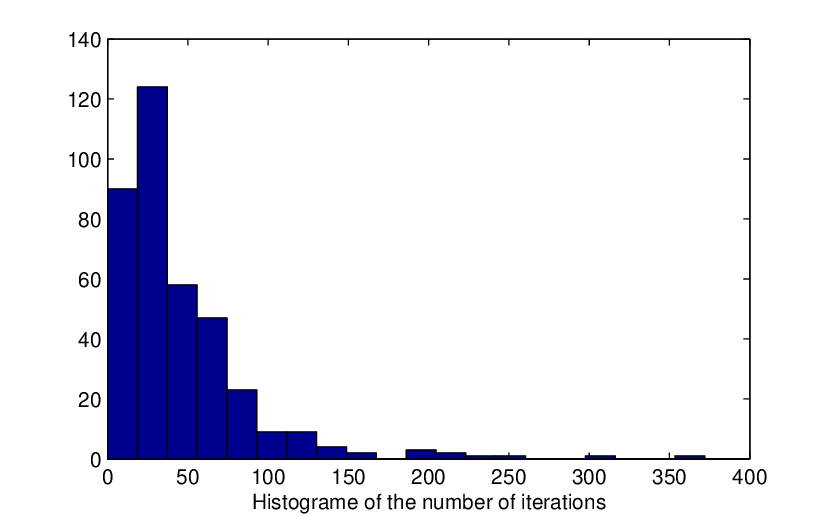}
\end{center}
\caption{The left graph shows the number of iterations
needed by SPMP for the approximation of the flute tone in 
Fig.~\ref{tone}.
 The right graph is the histogram of the values in the 
left graph.}
\label{katri2}
\end{figure}

\subsection{On the accuracy of self projections}

The numerical accuracy of most used 
 direct methods for 
calculating a projection is well studied 
\cite{Bjo96,GV96,Hig02, Wil65, Bjo67, Bjo94, GLE05,GLR05, Ste05} and 
 also the subject of recent research in particular contexts
 \cite{VVl10,Ste11,LBG13}. Contrarily, 
 the numerical analysis of the SPMP algorithm 
has not yet been addressed.  
Therefore,  this section discusses the accuracy of 
the self projection procedure, when implemented in 
finite precision arithmetic. 
 
Since the self projection steps (a) - (c) in Sec.~\ref{asp} 
are 
based on recursive calculation of inner products,  
we base the numerical analysis of the method on 
two basic results. 
As usual the evaluation of an arithmetic operation 
is denoted as $\fl(\cdot)$  and the unit roundoff
as $u$. Thus, for $\vf_1 \in \R^N$ and 
$\vf_2 \in \R^N$ the  numerical error
in the calculation of the inner product
$\la \vf_1, \vf_2 \ra$ is bounded as (\!\cite{GV96}, p. 99)
\begin{equation}
\label{eq:inner}
|\fl(\la \vf_1,\vf_2 \ra)- \la \vf_1,\vf_2 \ra|\leq N u \|\vf_1\|\|\vf_2\| + \O(u^2).
\end{equation}
The computation of the 
saxpy operation $\alpha \vf_1  + \vf_2$, with $\alpha$  a number,  is 
bounded as (\!\cite{GV96}, p. 100)
\be
\label{saxpy}
\|\fl(\alpha  \vf_1 +\vf_2)- (\alpha  \vf_1 +\vf_2) \|\leq 
u  (2\| \alpha\vf_1\| +\|\vf_2\|) + \O(u^2).
\ee
\begin{theorem}
An approximate bound for the error produced by 
 implementing the 
projection step in 
finite precision arithmetics is give as
\be
\label{errbound}
\|\Delta \bar \vr^{k,j}_T \| \lessapprox   (N+3) j u \|\vr^{k,0}\| + \O(u^2).
\ee
\end{theorem}
\begin{proof}
Denoting the computed quantities by $\bar \vr^{k, j}$   and by $ \bar l_j$ the indices selected with the computed quantities, using \eqref{saxpy} we have  
\be
\label{errmodel}
\bar \vr^{k, j}= \bar \vr^{k, j-1} - \fl( \la \bar \vr^{k,j-1}, \vd_{\bar l_j} \ra)   \vd_{\bar l_j}  +  \delta \bar \vr^{k,j},
\ee
with  
$$\| \delta \bar \vr^{k,j} \| \leq  u \left (\| \bar \vr^{k, j-1} \| + 2 | \fl( \la\bar \vr^{k,j-1}, \vd_{\bar l_j}  \ra)| \right ) + 
\O(u^2).
$$
Through straightforward manipulation we further have
\ben
\| \delta \bar \vr^{k,j} \|& \leq & 
u \left (\| \bar \vr^{k, j-1} \|+2 |\fl( \la\bar \vr^{k,j-1}, \vd_{\bar l_j} \ra) - \la\bar \vr^{k,j-1}, \vd_{\bar l_j}  \ra| + 2 |\la\bar \vr^{k,j-1}, \vd_{\bar l_j} 
 \ra|\right) + \O(u^2) \nonumber
\een
so that, using \eqref{eq:inner}, we finally 
obtain
\ben
\label{delr}
\| \delta \bar \vr^{k,j} \| 
& \leq & u \left (3 \| \bar \vr^{k, j-1} \| 
+  2N u  \| \bar \vr^{k, j-1} \| \right ) + \O(u^2)=
3u  \| \bar \vr^{k, j-1} \| + \O(u^2). 
\een
Moreover, \eqref{errmodel} can be rewritten as
\be
\label{model2}
\bar \vr^{k, j}= \bar \vr^{k, j-1} - \la \bar \vr^{k,j-1}, \vd_{\bar l_j} \ra  \vd_{\bar l_j}  +  \Delta \bar \vr^{k,j},
\ee
where
$\Delta \bar \vr^{k,j}= - \fl( \la \bar \vr^{k,j-1}, \vd_{\bar l_j} \ra)   \vd_{\bar l_j} + \la \bar \vr^{k,j-1}, \vd_{\bar l_j} \ra  \vd_{\bar l_j} +  \delta \bar \vr^{k,j}.$ Using 
now \eqref{delr} and \eqref{eq:inner} we have  the  
bound for the norm of $\Delta \bar \vr^{k,j}$ in the 
form
\ben
\label{Delr}
\|\Delta \bar \vr^{k,j}\| \le 
 N u \| \bar \vr^{k, j-1} \| +  3u \| \bar \vr^{k, j-1} \| + \O(u^2)=  u(N + 3)  \| \bar \vr^{k, j-1} \| + \O(u^2).
\een
Thus, due to  rounding errors instead of the theoretical 
result $\|\vr^{k,j}\| \le \| \vr^{k,j-1}\|$ we 
only have
$$\| \bar \vr^{k,j}\| \le (1+(N+3)u)\| \bar \vr^{k,j-1}\|  + \O(u^2)\le 
(1+(N+3)u)^{j} \| \bar \vr^{k,0}\| + \O(u^2).$$
This inequality gives rise to the recurrence for 
bounding the total error in the calculation of $\vr^{k,j}$. 
 In terms of the matrices 
$\bar\vT_i= (\vI - \vd_{\bar l_i} \vd_{\bar l_i}^\topc),\, 
i=1,\ldots,j$, 
where $\vI\in \R^{N \times N}$ is the identity matrix, 
equation \eqref{model2} can be expressed in the form
\be
\bar \vr^{k,j}= \bar\vT_j \bar\vT_{j-1} \cdots \bar \vT_1 \vr^{k,0} + \Delta\bar \vr^{k,j}_T, 
                   \nonumber 
\ee
where $\Delta \bar \vr^{k,j}_T=
 \sum_{i=1}^j    \bar\vT_{j} \dots \bar\vT_{i+1}   \Delta \bar \vr^{k,i}$   (with the notation $ \bar\vT_{j}\bar\vT_{j+1}=\vI$). Since $\| \bar\vT_{i} \| = 1$  for all  $i=1,\ldots,j$  and   $\Delta \bar \vr^{k,i}$    is
bounded as in \eqref{Delr},   it follows that
 $\Delta \bar \vr^{k,j}_T$ is
bounded as
\ben  
 \|\Delta \bar \vr^{k,j}_T\| &\le &
\sum_{i=1}^j \|\Delta \bar \vr^{k,i}\| \le 
u (N+3) \sum_{i=1}^j 
(1+(N+3)u)^{i} \|\vr^{k,0}\| + \O(u^2).
\een
Restricting considerations to 
$N u \ll 1$  we have the approximate bound
\be
\label{finalbound}
\|\Delta \bar \vr^{k,j}_T \| \lessapprox   (N+3) j u \|\vr^{k,0}\| + \O(u^2).
\ee
\end{proof}
Even  if, 
as discussed in Sec.~\ref{cr}, 
 in the limit $j \to \infty$  the 
convergence  $\vr^{k,j} \to  \vf -\op_{\V_k} \vf$ 
is theoretically guaranteed, the size  of $\Delta \bar  \vr^{k,j}_T$ gives a  
limit for the maximum number of recursive operations. 
  Beyond that 
limit the calculations in the self projection algorithm 
are dominated by rounding errors. However, 
in situations of practical interest the numerical 
convergence is fast enough 
for the algorithm  to operate within 
the boundary  of reliability  
 established in \eqref{finalbound}.
\section{Hierarchized Block Wise SPMP} 
\label{HBWSPMP}
The Hierarchized Block Wise (HBW) version of 
pursuit strategies is an
implementation of those techniques dedicated to 
approximating by partitioning. The method 
approximates each element of a signal partition 
independently of each other, but 
links the approximations by a 
global constraint on sparsity \cite{RNMB13,LRN16}. The 
 strategy proceeds simply by ranking the partition 
units for their sequential stepwise approximation. 
This section  discusses the 
HWB version of the OMP approach (HBW-OMP) 
\cite{RNMB13, LRN16} but  
implemented via the SPMP method (HBW-SPMP).

Let's suppose that a given  
signal $\vf$ is split 
into $Q$ disjoint `blocks' $\vf_q,\,q=1,\ldots,Q$, where 
each $\vf_q$ is an element of $\R^{\Nb}$, with $\Nb=N/Q$. 
Denoting by
$\oj$ the concatenation operator, the 
signal $\vf \in \R^N$ is `assembled' from the blocks as 
$\vf=\oj_{q=1}^Q \vf_q$. This operation implies that 
the  first $N_1$ components of  the vector $\vf$ are given 
by the vector $\vf_1$, the next $N_2$ components by the 
vector $\vf_2$ and so on. The HBW version of SPMP 
for approximating the signal's partition using $K$ atoms 
in total is implemented by the following steps. 
\begin{itemize}
\item[1)] For $q=1,\ldots,Q$ set
$\vr_q^0=\vf_q$, $ \vf_q^0=0$, $\Ga_0^q=\{\emptyset\}$ 
and $\kq=1$.  
Initialize the algorithm by
selecting the `potential' first atom for
 the atomic decomposition of every
block $q$, according to the MP criterion: 
$$\ell_{\kq}^q=
\operatorname*{arg\,max}_{n=1,\ldots,M} \left | \la \vd_n, 
\vr_q^{\kq-1} \ra \right|,\, q=1,\ldots,Q.$$
\item[2)] Select the block $\qs$ such that
$$\qs = \operatorname*{arg\,max}_{q=1,\ldots,Q} \left| 
\la \vd_{\ell^q_{\kq}}, \vr_q^{\kq-1} \ra\right|.$$
Update the set   
$\Ga^{\qs}_{\kqs}=\Ga^{\qs}_{\kqs-1} \cup \{\ell^{\qs}_{\kqs}\}$ 
and the atomic decomposition of the block $\qs$
 by incorporating the atom $\vd_{\ell^{\qs}_{\kqs}}$ 
 i.e., use $c^{\qs}\!(\kqs)=\la  \vd_{\ell^{\qs}_{\kqs}}, 
\vr_{\qs}^{\kqs-1}\ra$ to compute  
\ben
\vf_{\qs}^{\kqs}&=&\vf_{\qs}^{\kqs-1} + c^{\qs}\!(\kqs)\vd_{\ell^{\qs}_{\kqs}},
\nonumber\\
\vr_{\qs}^{\kqs} &=& \vf_{\qs} - \vf_{\qs}^{\kqs}.\nonumber
\een
If $\kqs>1$ set $\vr_{\qs}^{k,0}=\vr_{\qs}^{k}$ and 
starting from $j=1$ 
  realize the projection   
 as indicated below.
\begin{itemize}
\item [(a)]
Choose,
out of the set 
$\Ga^{\qs}_{\kqs}=\{\ell^{\qs}_i\}_{i=1}^{\kqs}$,
the index $l_j$ such that
$$l_j= \operatorname*{arg\,max}_{\substack{i=1,\ldots,\kqs}} \left|\la \vd_{\ell^{\qs}_i}, \vr^{\kqs,j-1} \ra\right|.$$
If $\left|\la \vd_{l_j}, \vr^{\kqs,j-1} \ra\right|< \eps$ 
jump to  3). Otherwise proceed with steps b) and c).
\item [(b)]
Use $\la \vd_{l_j} , \vr^{\kqs,j-1} \ra$  to
update the coefficient $c^{\qs}\!(l)$, 
the approximation
$\vf^{\kqs}_{\qs}$, and the residual  as
\ben
c^{\qs}\!(l_j)  &\leftarrow & c^{\qs}(l_j)+ \la \vd_{l_j}, \vr^{\kqs,j-1} \ra, \nonumber \\
\vf^{\kqs} &\leftarrow & \vf^{\kqs} +  \la \vd_{l_j} , \vr^{\kqs,j-1} \ra  \vd_{l_j},\nonumber\\
\vr^{\kqs,j}&=&\vr^{\kqs,j-1} - \la \vd_{l_j} , \vr^{\kqs,j-1} \ra  \vd_{l_j}. \nonumber 
\een
\item [(c)]
 Increment $j \leftarrow j+1$ and repeat steps
(a) $\to$ (c) until the stopping criterion is met.
\end{itemize}
\item[3)] Check if for the given number 
$K$ the stopping condition $\sum_{q=1}^Q \kq=K$ has been met.
Otherwise:
 \begin{itemize} 
 \item [$\bullet$]
 Increase $\kqs \leftarrow \kqs +1$.  
 \item [$\bullet$]
 Select a new potential atom for the
 atomic decomposition of block $\qs$
$$\ell_{\kqs}^{\qs}=
\operatorname*{arg\,max}_{n=1,\ldots,M} \left | \la \vd_n,
\vr_{\qs}^{\kqs-1} \ra \right|.$$
 \item [$\bullet$]
 Repeat 2) and 3).
\end{itemize}
\end{itemize}
\subsection{Numerical Example II}
\label{RWS}
We construct here the atomic decomposition of the Pop Piano 
and Classic Guitar clips shown in 
Fig.~\ref{clips}.  
Both clips consists of $N= 262144$ samples 
at 44100Hz each (5.94 secs length).  For the approximation
 we use the trigonometric
dictionary $\D^{cs}$ introduced in Sec.~\ref{NE1}.

\begin{figure}[!ht]
\begin{center}
\includegraphics[width=13cm]{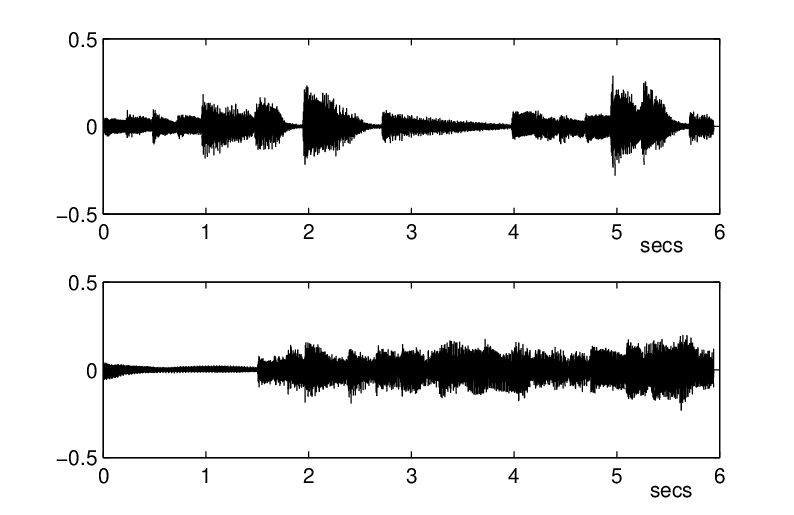}
\end{center}
\caption{\small{Pop Piano (top graph)
and Classic Guitar music signals.
Both clips consist of $N= 262144$ samples
at 44100Hz each (5.94 secs length).}}
\label{clips}
\end{figure}

The global sparsity of the signal approximation
is measured by the Sparsity Ratio
(SR) which is defined as
$\ds{\text{SR}= \frac{N}{K}}$, where $K$ is  
  the total number of coefficients in the signal 
representation. Hence,
the larger the value of SR is the smaller the number of
frequency components needed for the approximation.

The sparsity results of the clips in Fig.~\ref{clips} 
are shown in Fig.~\ref{SR1}, for the 
MP, HBW-MP, SPMP, HBW-SPMP approaches 
 and partitions of unit size $\Nb$ 
equal to 1024, 2048, 4096, and 8192 samples. 
For larger values of $\Nb$
the sparsity does not improve significantly. The quality of 
the approximation is fixed to yield a SNR of 35dB.
 As observed in 
Fig.~\ref{SR1} for the two clips in Fig.~\ref{clips} 
the gain in sparsity achieved by implementing the SPMP 
approach in the HBW manner is significant.

\begin{figure}[!ht]
\begin{center}
\includegraphics[width=8cm]{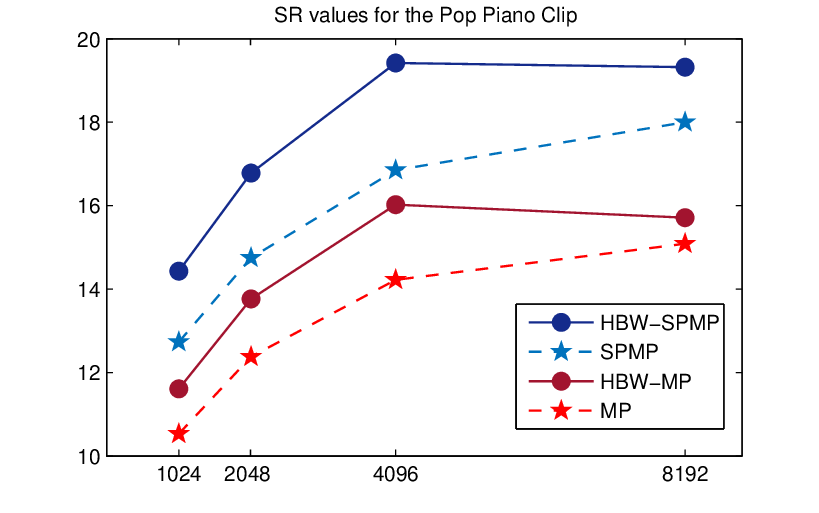}
\includegraphics[width=8cm]{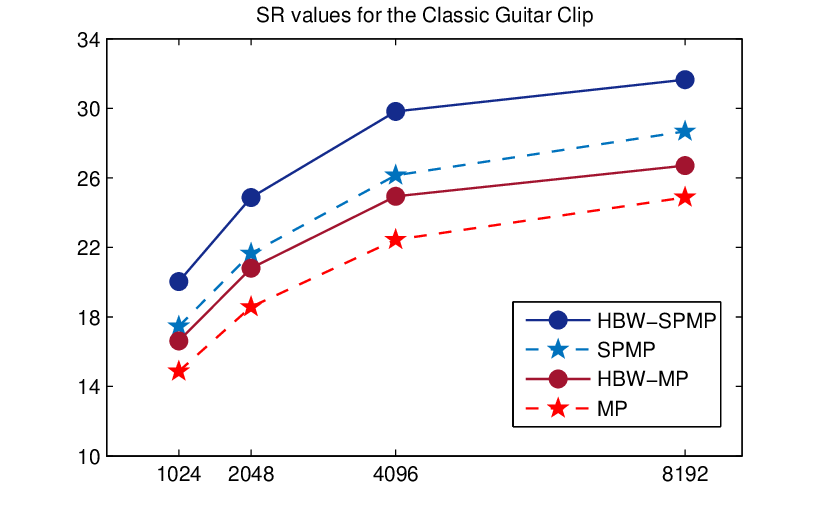}
\end{center}
\caption{\small{SR vs partition unity size $\Nb=
1024, 2048, 4096$ and 8192 samples for the music clips of 
Fig.~\ref{clips}}. The graph on left corresponds to the 
Pop Piano and the other to the Classic Guitar.}
\label{SR1}
\end{figure}

{\bf{Note:}} The MATLAB function HBW-SPMP 
dedicated to reproducing the above example with 
the trigonometric dictionary $\D^{cs}$,
via the FFT, 
has been have been made available on \cite{webpage}.
The MATLAB and C++ codes for implementing SPMP 
 with general dictionaries, as well as the 
corresponding SPMP2D versions for separable dictionaries 
are available on \cite{webpage1}. The MATLAB and C++ codes 
 for SPMP3D can be found on \cite{webpage2}.

\section{Conclusions}
\label{Con}
The  convergence rate of the SPMP 
algorithm, which implements the OMP  greedy strategy 
 by means of the MP one, was derived.  
The orthogonal projection step, 
intrinsic to the OMP method, is realized within the 
SPMP framework by  subtraction from the 
residual error its approximation using the MP 
algorithm with a dictionary consisting only of the 
already selected atoms, up to the particular step. Thus,  
the memory requirements are kept within 
the same scale as for MP. The bound for the 
self projection convergence 
rate (c.f. \eqref{cr2}) clearly highlights
 the broad range of cases for which 
the OMP greedy strategy can be implemented 
through the  SPMP method.  The cases for 
which the convergence could become very slow fall within the 
 class of ill posed problems. 

The analysis of the accuracy of the 
 projection step, when implemented in finite precision 
arithmetics, produced a meaningful upper bound 
 relating the number of 
iterations with the dimension of system and the 
unit roundoff.  This worst-case behavior bound confirms that the SPMP method 
is suitable to be applied to solve 
 well posed problems for which the convergence
is fast. Otherwise, as the number of iterations 
increases the accuracy of the approach would be 
dominated by roundoff errors. Nevertheless, a number 
of applications to real  world 
signals \cite{RNB13, RNA16, LRN17, RNW19} 
have already confirmed that 
the approach is of assistance 
for practical implementations of the OMP greedy strategy 
in situations where, due to memory requirements, 
direct linear algebra techniques cannot be 
applied.

The HBW extension of a pursuit strategy for 
approximating a signal partition was 
considered in relation to the SPMP implementation 
for reduction in memory requirements.   
The suitability of the technique was highlighted by
numerical tests which, due to memory limitations,  
 could not have been realized in a standard computer 
by other implementations of OMP.

\subsection*{Acknowledgements}
We are grateful to anonymous Reviewers for the careful 
reading of the paper and their helpful comments 
 and constructive remarks. 
\bibliographystyle{elsart-num}
\bibliography{revbib}
\end{document}